\documentclass[10pt,twocolumn,letterpaper]{article}

\usepackage{wacv}
\usepackage{times}
\usepackage{epsfig}
\usepackage{graphicx}
\usepackage{amsmath}
\usepackage{amssymb}

\usepackage{subfigure}
\usepackage{graphicx}
\usepackage[inline]{enumitem}
\usepackage{amsthm}
\usepackage{algorithm}
\usepackage{algorithmic}

\newtheorem{theorem}{Theorem}
\theoremstyle{definition}

\usepackage{booktabs}
\usepackage{wrapfig}
\usepackage{lipsum}
\usepackage{booktabs} 
\usepackage{adjustbox}
\usepackage{mathtools}
\usepackage{comment}

\DeclareMathOperator{\E}{\mathbb{E}}

\wacvfinalcopy 

\def\wacvPaperID{427} 

\ifwacvfinal\fi
\begin{document}

\title{microbatchGAN: Stimulating Diversity with Multi-Adversarial Discrimination}
\author{Gon\c{c}alo Mordido \hspace{2cm} Haojin Yang \hspace{2cm} Christoph Meinel\\
Hasso Plattner Institute\\
{\tt\small goncalo.mordido@hpi.de}
}

\maketitle
\ifwacvfinal\thispagestyle{empty}\fi

\begin{abstract}
We propose to tackle the mode collapse problem in generative adversarial networks (GANs) by using multiple discriminators and assigning a different portion of each minibatch, called microbatch, to each discriminator. We gradually change each discriminator's task from distinguishing between real and fake samples to discriminating samples coming from inside or outside its assigned microbatch by using a diversity parameter $\alpha$. The generator is then forced to promote variety in each minibatch to make the microbatch discrimination harder to achieve by each discriminator. Thus, all models in our framework benefit from having variety in the generated set to reduce their respective losses.
We show evidence that our solution promotes sample diversity since early training stages on multiple datasets.
\end{abstract}

\section{Introduction}

Generative adversarial networks~\cite{gans}, or GANs, consist of a framework describing the interaction between two different models - one generator (G) and one discriminator (D) - that are trained together. 
While $G$ tries to learn the real data distribution by generating realistic looking samples that are able to fool $D$, $D$ tries to do a better job at distinguishing between real and the fake samples produced by $G$.
Although showing very promising results across various domains~\cite{edwards2015censoring, ho2016generative, yu2017seqgan, yang2017improving, donahue2018synthesizing}
, GANs have also been continually associated with instability in training, more specifically mode collapse~\cite{cross_domain_gans, towards, the_numerics_of_gans, mode_regularized_gans, wgan}.
This behavior is observed when $G$ is able to fool $D$ by only generating samples from the same data mode, leading to very similar looking generated samples. This suggests that $G$ did not succeed in learning the full data distribution but, instead, only a small part of it. This is the main problem we are trying to solve with this work.

The proposed solution is to use multiple discriminators and assign each $D$ a different portion of the real and fake minibatches, \textit{i.e., }microbatch. Then, we update each $D$'s task to discriminate between samples coming from its assigned fake microbatch and samples from the microbatches assigned to the other discriminators, together with the real samples. We call this microbatch discrimination. Throughout training, we gradually change from the originally proposed real and fake discrimination by \cite{gans} to the introduced microbatch discrimination by the use of an additional diversity parameter $\alpha$ that ultimately controls the diversity in the overall minibatch.

The main idea of this work is to force $G$ to reduce its loss by inducing variety in the generated set, complicating each $D$'s task on separating the samples in its microbatch from the rest. Even though only producing very similar images would also complicate the desired discrimination, it would not benefit any of the models. This is due to the attribution of distinct probabilities by each $D$ to samples from and outside its microbatch being required to minimize $G$ and $D$'s losses. Hence, all models in the proposed framework, called microbatchGAN, benefit directly from diversity in the generated set.

Our main contributions can be stated as follows:
\begin{enumerate*}[label=(\roman*)]
\item proposal of a novel multi-adversarial GANs framework (Section~\ref{sec:solution}) that mitigates the inherent mode collapse problem in GANs;
\item empirical evidence on multiple datasets showing the success of our approach in promoting sample variety since early stages of training (Section~\ref{sec:results})
\item Competitiveness against other previously proposed methods on multiple datasets and evaluation metrics (Section~\ref{sec:comparision}).
\end{enumerate*}

\subsection{Related Work}

Previous works have optimized GANs training by changing the overall models' objectives, either by using discrepancy measurements~\cite{mmd_gan, sutherland2016generative} or different divergence functions~\cite{nowozin2016f, uehara2016generative} to approximate the real data distribution. Moreover, \cite{energy_gan, berthelot2017began, coulomb_gan} proposed to use energy-driven objective functions to encourage sample variety, \cite{mroueh2017mcgan} tried to match the mean and covariance of the real data, and \cite{unrolled_gan} used an unrolled optimization of $D$ to train $G$. 
\cite{mode_regularized_gans, warde2016improving, wang2017magan, berthelot2017began} penalized missing modes by using an extra autoenconder in the framework. \cite{is} performed minibatch discrimination by forcing $D$ to condition its output on the similarity between the samples in the minibatch. \cite{springenberg2015unsupervised} increased $D$'s robustness by maximizing the mutual information between inputs and corresponding labels, while \cite{lin2017pacgan} forced $D$ to make decisions over multiple samples of the same class, instead of independently.

Regarding using multiple discriminators, \cite{multiple_random_projection_gans} extended the framework to several discriminators with each focusing in a low-dimensional projection of the data, set a priori. \cite{gman} proposed GMAN, consisting of an ensemble of discriminators that could be accessed by the single generator according to different levels of difficulty.
\cite{dual_GANs} introduced D2GAN, introducing a single generator dual discriminator architecture where one discriminator rewards samples coming from the true data distribution whilst the other rewards samples coming from the generator, forcing the generator to continuously change its output.
\cite{mordido2018dropout} proposed Dropout-GAN, applying adversarial dropout by omitting the feedback of a given $D$ at the end of each batch.

\section{Generative Adversarial Networks}

The original GANs framework~\cite{gans} consists of two models: a generator ($G$) and a discriminator ($D$). Both models are assigned different tasks: whilst $G$ tries to capture the real data distribution $p_r$, $D$ learns how to distinguish real from fake samples. $G$ maps a noise vector $z$, retrieved from a noise distribution $p_{z}$, to a realistic looking sample belonging to the data space. $D$ maps a sample to a probability $p$, representing the likeliness of that given sample coming from $p_{r}$ rather than from $p_g$. The two models are trained together and play the following minimax game:

\begin{equation}
\begin{split}
\min_{G}\max_{D}V(D,G) =\\ \E_{x \sim p_{r}(x)}[\log D(x)] + \E_{z \sim p_{z}(z)}[\log (1-D(G(z)))],
\end{split}
\label{gans_minimax}
\end{equation}

where $D$ maximizes the probability of assigning samples to the correct distribution and $G$ minimizes the probability of its samples being considered from the fake data distribution. 

Alternatively, one can also train $G$ to maximize the probability of its output being considered from the real data distribution, \textit{i.e.,} $\log D(G(z))$. Even though this changes the type of the game, by being no longer minimax, it avoids the saturation of the gradient signals at the beginning of training~\cite{gans}, where $G$ only receives continuously negative feedback, making training more stable in practice. However, since we employ multiple discriminators in the proposed framework, it is less likely that $G$ does not receive any positive feedback from the whole adversarial ensemble~\cite{gman}. Therefore, we make use of the original value function in this work.

\section{microbatchGAN}
\label{sec:solution}

In this work, we propose a novel generative multi-adversarial framework named microbatchGAN, where we start by splitting each minibatch into several microbatches and assigning a unique one to each $D$. The key aspect of this work is the usage of microbatch discrimination, where we change the original discrimination task of distinguishing between real and fake samples, as proposed in \cite{gans}, to each $D$ distinguishing between samples coming or not from its fake microbatch. This change is performed in a gradual fashion, using an additional diversity parameter $\alpha$. Thus, each $D$'s output gradually changes from the probability of a given sample being real to the probability of a given sample not belonging to its fake microbatch. Moreover, since each $D$ is trained with different fake and real samples, we encourage them to focus on different data properties. Figure~\ref{fig:frameworks} illustrates the proposed framework.


\begin{figure}[h]
  \begin{center}
    \includegraphics[width=0.45\textwidth]{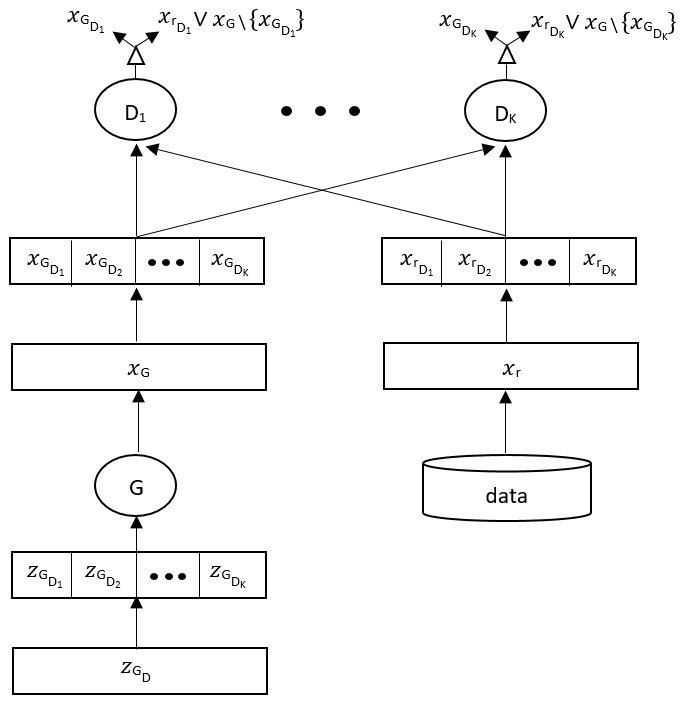}
  \end{center}
  \caption{microbatchGAN framework assuming a positive diversity parameter $\alpha$. Each discriminator $D_k$ is assigned a different microbatch $x_{G_{D_{k}}}$, where it discriminates between samples coming from inside its microbatch and samples coming from the microbatches assigned to the rest of the discriminators ($x_G \setminus x_{G_{D_{k}}}$) together the real samples $x_{r_{D_{k}}}$.}
  \label{fig:frameworks}
\end{figure}

The proposed microbatch-level discrimination task leads to $G$ making such discrimination harder for each $D$ to lower its loss. Hence, $G$ is forced to induce variety on the overall minibatch, making it a substantially harder task for each $D$ to be able to separate its subset of fake samples in the diverse minibatch. Note that producing very similar samples across the whole minibatch would also make such discrimination difficult by making the whole minibatch the same. However, $G$ also benefits from each $D$ assigning distinct probabilities to samples from inside and outside its designed microbatch to lower its loss, making the generation of different samples in the minibatch a necessary requirement to obtain different outputs from $D$. Hence, all models in our framework benefit directly from sample variety in the generated set.



In the microbatchGAN scenario with a positive diversity parameter $\alpha$, each $D$ assigns low probabilities to fake samples from its microbatch and high probabilities to fake samples from the rest of the microbatches as well as samples from the real data distribution. Hence, fake samples in the rest of the minibatch, \textit{i.e., }not coming from its assigned microbatch, shall be given distinct output probabilities by each $D$. On the other hand, $G$ minimizes the probability given by each $D$ to the samples outside its microbatch and maximizes the probability given to the fake samples assigned to that specific $D$. The value function of our minimax game is as follows:

\begin{equation}
\begin{split}
\min_{G}\max_{\big\{D_k\big\}} \sum_{k=1}^{K} V(D_k,G) =\sum_{k=1}^{K} \E_{x \sim p_{r_{D_{k}}}(x)}[\log D_k(x)] \\+ \E_{z \sim p_{z_{G_{D_{k}}}}(z)}[\log (1-D_k(G(z)))] \\+ \alpha \times \E_{z' \sim p_{z_{G_{D}}\setminus \small\{z_{G_{D_{k}}}\small\}}(z')}[\log D_{k}(G(z'))],
\end{split}
\label{ours_minimax}
\end{equation}

where K represents the number of total discriminators in the set. $p_{r_{D_{k}}}$ represents real samples from $D_k$'s real microbatch,  $p_{z_{G_{D_{k}}}}$ indicates fake samples from $D_k$'s fake microbatch, and $p_{z_{G_{D}}\setminus \small\{z_{G_{D_{k}}}\small\}}$ relates to the rest of the fake samples in the minibatch but not in $p_{z_{G_{D_{k}}}}$. $\alpha$ represents the \textit{diversity parameter} responsible for penalizing the incorrect discrimination of fake samples coming from $p_{z_{G_{D}}\setminus \small\{z_{G_{D_{k}}}\small\}}$ by each $D_k$. Note that $\alpha$ = 0 would represent the original GANs objective for each $D$ in the set.
The training procedure of microbatchGAN is presented in Algorithm~\ref{alg:framework}.

\begin{algorithm}
    \caption{microbatchGAN.}
\begin{algorithmic}
   \STATE {\bfseries Input:} $K$ number of discriminators, $\alpha$ diversity parameter, $B$ minibatch size
    \STATE {\bfseries Initialize:} $m \leftarrow \frac{B}{K}$
    
    \FOR{number of training iterations}
    
    \STATE {} \textbullet~Sample minibatch $z_i$, $i=1 \ldots B$, $z_i \sim p_g(z)$
    \STATE {} \textbullet~Sample minibatch $x_i$, $i=1 \ldots B$, $x_i \sim p_{r}(x)$
    
    \FOR{$k = 1$ to $k = K$}
    
    \STATE {} \textbullet~Sample microbatch $z_{k_{j}}$, $j = 1 \ldots m$, $z_{k_{j}} = z_{(k-1) \times m + 1 : k \times m}$
    \STATE {} \textbullet~Sample microbatch $x_{k_{j}}$, $j = 1 \ldots m$, $x_{k_{j}} = x_{(k-1) \times m + 1 : k \times m}$
    \STATE {} \textbullet~Sample microbatch $z'_{k_{j}}$, $j = 1 \ldots m$,  $z'_{k_{j}} \subset z_i \setminus \small\{z_{k_{j}}\small\}$
        
    
    
    \STATE {} \textbullet~Update $D_k$ by ascending its stochastic gradient:
    $$\nabla_{\theta_{D_k}} \frac{1}{m} \sum_{j=1}^{m} [\log D_k(x_{k_{j}}) +\log (1-D_k(G(z_{k_{j}})))$$
    $$+ \alpha \times \log D_{k}(G(z'_{k_{j}}))]$$
    \ENDFOR
    \STATE {} \textbullet~Update $G$ by descending its stochastic gradient:
    $$\nabla_{\theta_{G}} \sum_{k=1}^{K} \big[ \frac{1}{m} \sum_{j=1}^{m} [\log (1-D_k(G(z_{k_{j}})))$$
    $$+ \alpha \times \log D_{k}(G(z'_{k_{j}}))]\big]$$
    
    \ENDFOR
\end{algorithmic}
    \label{alg:framework}
\end{algorithm}

\subsection{Theoretical Discussion}

To better understand how our approach differs from the original GANs in promoting variety in the generated set, we study a simplified version of the minimax game where we freeze each $D_k$ and train $G$ until convergence. In the most extreme case, we say that we have mode collapse when:

\begin{equation}
\label{proof_mode_collapse}
\text{For all } z' \sim p_g(z), G(z')=x
\end{equation}

\begin{theorem}
In original GANs, mode collapse fully minimizes $G$'s loss when we train $G$ exhaustively without updating $D$.
\end{theorem}

\begin{proof}
The optimal $x^{\ast}$ is the one that maximizes $D$'s output, where:
$x^{\ast} = \underset{x}{\text{argmax}} D(x)$.
Thus, assuming $G$ would eventually learn how to produce $x^{\ast}$, mode collapse on $x^{\ast}$ would fully minimize its loss, making $x^{\ast}$ independent of $z$.
\end{proof}

\begin{theorem}
In microbatchGAN, assuming $\alpha > 0$, $x \sim p_g$ must be dependent of $z$ for $G$ to fully minimize its loss, mitigating mode collapse when we train $G$ exhaustively without updating any $D_k$.
\end{theorem}

\begin{proof}

From Eq.~\ref{ours_minimax}, the value function between $G$ and each $D_k$ can be expressed as

\begin{equation}
\label{proof_1}
\begin{split}
V(D_k,G) = \E_{x \sim p_r}[\log D_k(x)] + \E_{x' \sim p_g}[\log (1-D_k(x'))] \\+ \alpha \times \E_{x'' \sim p_g}[\log D_{k}(x'')].
\end{split}
\end{equation}

To fully minimize its loss in relation to $D_k$, $G$ must find

\begin{equation}
    x' = \underset{x}{\text{argmax}} D_k(x) \text{ and } x'' = \underset{x}{\text{argmin}} D_k(x),
\end{equation}


which implies

\begin{equation}
    D_k(x') \ne D_k(x'') \implies x' \ne x''.
\end{equation}


Thus, generating different outputs for different $z$ is a requirement to fully minimize $G$'s loss regarding each $D_k$. Since we sum all $V(D_k,G)$ to calculate $G$'s final loss, this also applies to overall adversarial set, concluding the proof.

\end{proof}

\subsection{Diversity Parameter $\alpha$}

We control the weight of the microbatch discrimination in the models' losses by introducing an additional diversity parameter $\alpha$. Lower $\alpha$ values lead to $G$ significantly lowering its loss by generating realistic looking samples on each microbatch without taking much consideration on the variety of the overall minibatch. On the other hand, higher $\alpha$ values induce a stronger effect on $G$'s loss if each $D$ is able to discriminate between samples inside and outside its microbatch. However, high values of $\alpha$ might compromise the realistic properties of the produced samples, since too much weight is given to the last part of Eq.~\ref{ours_minimax}, being sufficient to effectively minimize $G$'s loss. Thus, using $\alpha > 0$ represents an additional way of ensuring data variety within the minibatch produced by $G$ at each iteration. An overview of different possible $\alpha$ settings follows below.

\bigbreak
\textbf{Static $\alpha$.}
First, we statically set $\alpha$ to values between 0 and 1 throughout the whole training. For the evaluation of the effects of each $\alpha$ value, we used a toy experiment of a 2D mixture of 8 Gaussian distributions (representing 8 data modes) firstly presented by~\cite{unrolled_gan}, and further adopted by~\cite{dual_GANs}. We used 8 discriminators for all the experiments. Results are shown in Figure~\ref{fig:static_alpha}.


\begin{figure}[h]
  \begin{center}
    \includegraphics[width=0.48\textwidth]{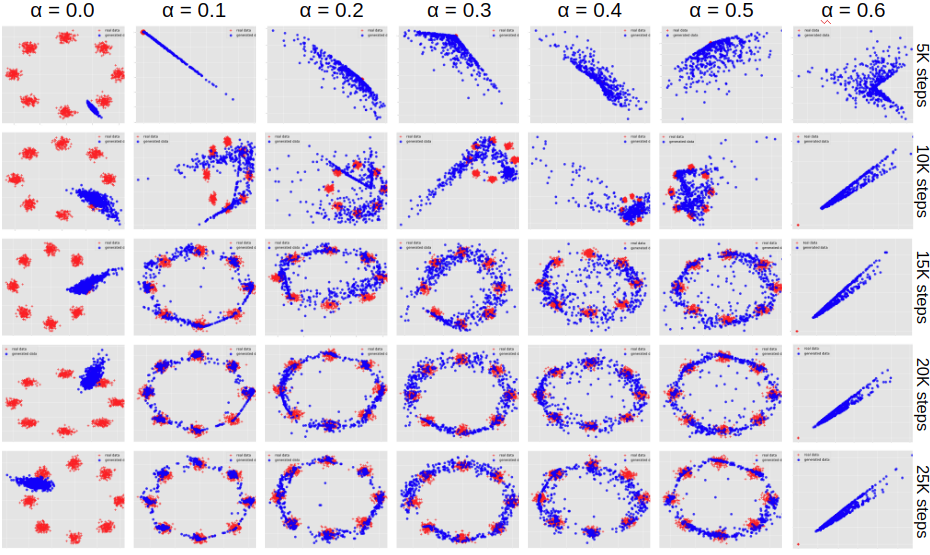}
  \end{center}
  \caption{Toy experiment using static $\alpha$ values. Real data is presented in red while generated data is in blue.}
  \label{fig:static_alpha}
\end{figure}

When setting $\alpha = 0$, $G$ mode collapses on a specific mode, showing the importance of using positive $\alpha$ values to mitigate mode collapse. When setting $0.1 \leq \alpha \leq 0.5$, $G$ is able to capture all data modes during training. However, learning problems in the early stages are observed, with $G$ only focusing on promoting variety in the generated samples. For higher $\alpha$ values ($\alpha  \geq 0.6$), $G$ was unable to produce any realistic looking samples throughout the whole training, focusing solely on sample diversity to lower its loss, suggesting the dominance of the last part of Eq.~\ref{ours_minimax}. Hence, a mild, dynamic, manipulation of $\alpha$ values seems to be necessary for a successful training of $G$, ultimately meaning both realistic and diverse samples from an early training stage.


\bigbreak
\textbf{Self-learned $\alpha$.}
We dynamically set $\alpha$ over time by adding it as a parameter of $G$ and letting it self-learn its values to lower its loss. However, we observed that $G$ takes advantage of being able to reduce its loss by increasing $\alpha$ at a large rate, focusing simply on promoting diversity in the generated samples without much realism, similarly to what was observed when using $\alpha = 0.6$ in the toy experiment (Figure~\ref{fig:static_alpha}). Hence, we suggest several properties that $\alpha$ should have so that diversity does not compromise the veracity of the generated samples. 

First, $\alpha$ should be upper bounded so that the last part of Eq.~\ref{ours_minimax} (responsible for sample diversity) does not overpower the first part (responsible for sample realism), ultimately not compromising the feedback given to $G$ to also be able to generate realistic samples. Second, $\alpha$'s growth should saturate over time, meaning that continuously increasing at large rates $\alpha$ is no longer an option to substantially decrease $G$'s loss over time. Lastly, to tackle the problem in learning of early to mid stages, we suggest that $\alpha$ should grow in a controlled fashion, so focus can also be given in the realistic aspect of the samples since the beginning of training.

Thus, we propose to make $\alpha$ a function of $\beta$, where $\alpha(\beta) \in \small[0,1\small[$, and let $G$ regulate $\beta$ instead of directly learning $\alpha$. We evaluated regulating $\alpha$ over three different functions that have the desired properties:

\begin{equation}
\alpha(\beta) =
    \begin{dcases*}
    \alpha_{sigm}(\beta) = Sigmoid(\beta), \beta \geq \beta_{sigm}\\
    \alpha_{soft}(\beta) = Softsign(\beta), \beta \geq \beta_{soft}\\
    \alpha_{tanh}(\beta) = Tanh(\beta), \beta \geq \beta_{tanh}
    \end{dcases*}
\label{gans_minimax_proposed_2}
\end{equation}




with $\beta_{sigm}$, $\beta_{soft}$, and $\beta_{tanh}$ representing the initial values of $\beta$ when training begins for the respective functions. For all the experiments of this paper, we set $\beta_{tanh} = \beta_{soft} = 0$, to obtain a positive codomain, and $\beta_{sigm} = -1.8$, since we achieved better empirical results by starting $\beta$ with this value (for further discussion about the effects of using different $\beta_{sigm}$ on $\alpha_{sigm}(\beta)$'s growth please see the Appendix). Note that learning $\alpha$ without any constraints can be characterized as using the identity function ($\alpha(\beta) = \alpha_{ident}(\beta) = \beta$).  Thus, each used function promotes a different $\alpha$ growth over time. To ease presentation, we neglect to write $\beta$'s dependence for the rest of the manuscript and use only the function names to described each $\alpha$ setting: $\alpha_{sigm}$, $\alpha_{soft}$, $\alpha_{tanh}$, and $\alpha_{ident}$.


\begin{figure}[h]
\flushright
  \subfigure[Generated samples.]{\includegraphics[width=.304\textwidth]{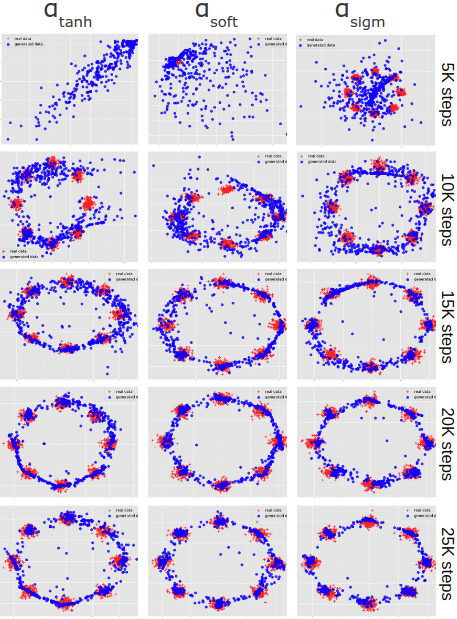}}
\begin{raggedleft}
  \subfigure[$\alpha$ evolution.]{\includegraphics[width=.158\textwidth]{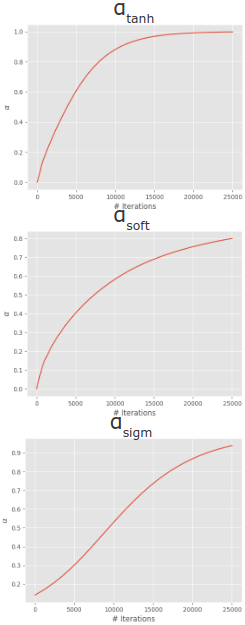}}
  \end{raggedleft}
\caption{Analysis of using different $\alpha$ functions on the toy dataset. The generated samples are shown in (a). The evolution of $\alpha$ on each function is presented in (b).}
\label{fig:tanh_softsign_sigmoid_toy_dataset}
\end{figure}

Results on the toy dataset using the different proposed $\alpha$ functions are shown in Figure~\ref{fig:tanh_softsign_sigmoid_toy_dataset}. The benefits of increasing $\alpha$ in a milder fashion, as performed when using $\alpha_{sigm}$, are observed especially early on training, with $G$ being concerned with the realism of the generated samples. On the other hand, when using $\alpha_{tanh}$ and $\alpha_{soft}$, the network takes longer to focus on the data realism (10K steps) since it is able to reduce its loss significantly by simply promoting variety due to the steeper growth of $\alpha$ in the earlier stages on both functions. Nevertheless, as the functions gradually saturate, all $\alpha$ settings manage to eventually capture the real data distribution while still keeping the diversity in the generated samples.

In conclusion, one can summarize microbatchGAN's training using these variations of self-learned $\alpha$ as the following: in the first iterations, $G$ increases $\alpha$ to reduce its loss, expanding its output. As $\alpha$ starts to saturate and each $D$ learns how to distinguish between real and fake samples, $G$ is forced to lower its loss by creating both realistic and diverse samples. 

\section{Experimental Results}
\label{sec:results}

We validated the effects of using different $\alpha$ functions on MNIST~\cite{mnist}, CIFAR-10~\cite{cifar10}, and cropped CelebA~\cite{celebA}. To quantitatively evaluate such effects, we used the Fr\'echet Inception Distance~\cite{fid}, or FID, since it has been shown to be sensitive to image quality as well as mode collapse~\cite{are_all_created_gans_equal}, with the returned distance increasing notably when modes are missing from the generated data.
We used several variations of the standard FID for a thorough study of $\alpha$'s effects in training, as well as the influence of using a different number of discriminators in our framework.

\subsection{Intra FID}

To measure the variety of samples of the generated set, we propose to calculate the FID between two subsets of 10K randomly picked fake samples generated at the end of every thousand iterations. We call this metric Intra FID. Important to note that Intra FID only measures the diversity in the generated set, not its realism. Hence, higher values indicate more diversity within the generated samples while lower values might indicate mode collapse in the generated set. The relation between Intra FID and progressive values of $\alpha$ is shown in Figure~\ref{fig:intra_fid_results}.

\begin{figure*}[h]
\centering
\includegraphics[width=0.90\textwidth]{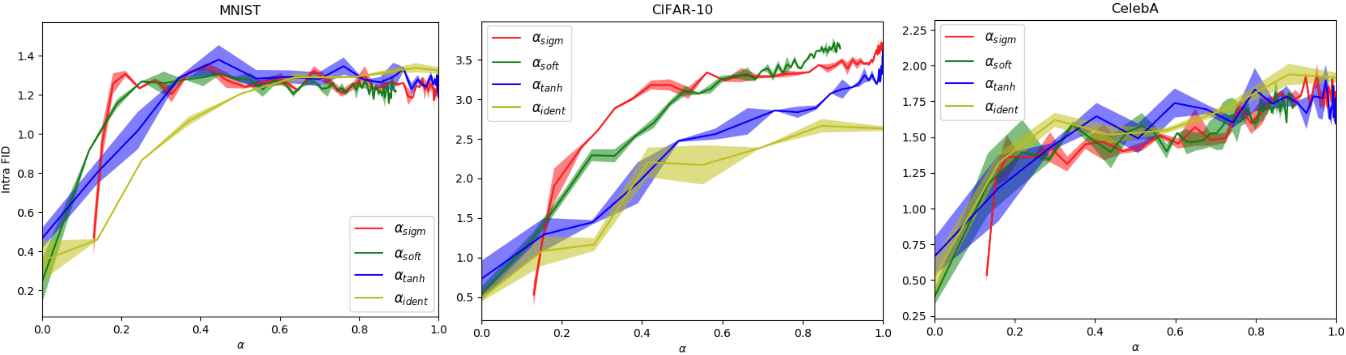}
\caption{Intra FID as $\alpha$ progresses. Higher values represent higher variety in the generated set.}
\label{fig:intra_fid_results}
\end{figure*}

We observe a strong correlation between $\alpha$'s growth and variety in the set, especially in beginning to mid-training. Later on, as $\alpha$ saturates, the variety is kept (represented by the stability of the Intra FID). It is further visible that $\alpha_{sigm}$, $\alpha_{soft}$, and $\alpha_{tanh}$ converge to similar Intra FID on all datasets. Important to note, that, to ease the visualization, the graphs only represent $0 \leq \alpha \leq 1$, with $\alpha_{ident}$'s values naturally surpassing 1 as time progresses.

\subsection{Cumulative Intra FID}

To analyze the sample variety over time, we summed the Intra FID values obtained from every thousand iterations. Hence, higher values indicate that the model was able to promote more variety in the set across time. Results are shown in Figure~\ref{fig:cum_intra_fid_results}, where we observe that using more discriminators leads to more variety across all datasets and $\alpha$ functions. Moreover, using $\alpha = 0$ leads to lower variety compared to using positive $\alpha$ values, with $\alpha_{sigm}$, $\alpha_{soft}$, and $\alpha_{tanh}$ obtaining similar values throughout the different datasets. Even though $\alpha_{ident}$ promotes the highest variety, the generated samples lack realism, as previously witnessed in the toy experiment and further discussed next.

\begin{figure*}[h]
\centering
\includegraphics[width=0.8\textwidth]{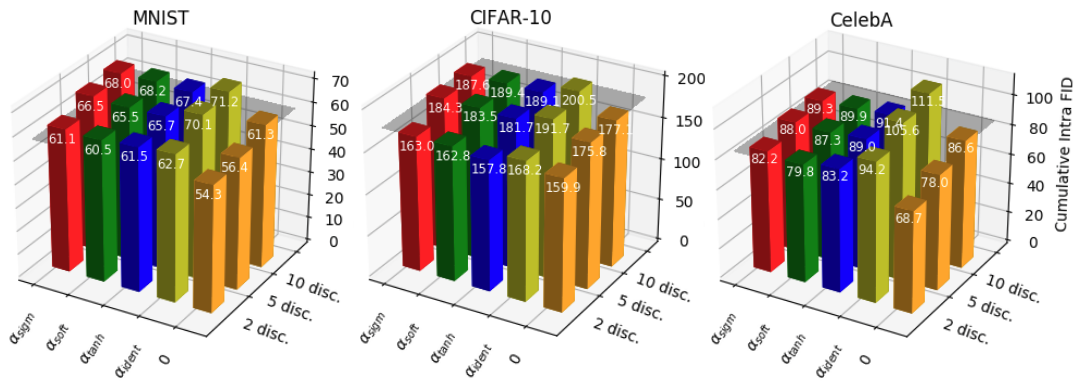}
\caption{Cumulative Intra FID using a different number of discriminators and $\alpha$ functions on the different datasets. Higher values correlate to higher variety in the produced samples across time. Values obtained using standard GANs are represented by the grey plane as a baseline.}
\label{fig:cum_intra_fid_results}
\end{figure*}

\subsection{Mean and Minimum FID}

To analyze both the realism and variety of the generated samples, we used the standard FID calculated between 10K fake samples and the real training data. Lower values should indicate both diversity and high-quality samples. The Mean FID and Minimum FID across 50K iterations are presented in Table~\ref{mean_min_fid_table} for each dataset. We observe that the best values, both in terms of mean and minimum, are obtained when using a higher number of discriminators, \textit{i.e., }5 or 10, and $\alpha_{tanh}$, $\alpha_{soft}$, and $\alpha_{sigm}$. Moreover, the high distances obtained when using $\alpha_{ident}$ confirm the lack of realism of the generated samples, highlighting the importance of constraining $\alpha$ by the properties previously stated in Section~\ref{sec:solution}.

\begin{table*}[h]
\vskip 0.15in
\begin{center}
\begin{small}
\begin{sc}
\centering
\begin{adjustbox}{width=0.8\textwidth}
\begin{tabular}{lc|cccccr}
\toprule
\multicolumn{2}{c|}{microbatchGAN} & \multicolumn{2}{c}{MNIST} & \multicolumn{2}{c}{CIFAR-10} & \multicolumn{2}{c}{CelebA}\\
\cmidrule(lr){1-2}\cmidrule(lr){3-4}\cmidrule(lr){5-6}\cmidrule(lr){7-8}
{K} & {$\alpha$} & {Mean FID} & {Min FID}  & {Mean FID} & {Min FID} & {Mean FID} & {Min FID} \\
\midrule
1 & -  & 50.9  $\pm$ 9.7 & 22.7 $\pm$ 0.7 & 125.5 $\pm$ 1.5 & 84.8 $\pm$ 1.6 & 77.3 $\pm$ 1.7 & 38.5 $\pm$ 1.1\\
\midrule
2 & $\alpha_{sigm}$ & \underline{37.6 $\pm$ 1.1} & \underline{23.5 $\pm$ 3.0} &  111.9 $\pm$ 0.1 & 90.8 $\pm$ 0.6 & 76.3 $\pm$ 0.6 & 53.0 $\pm$ 2.6\\
2 & $\alpha_{soft}$ & 41.9 $\pm$ 1.2 & 24.6 $\pm$ 0.0 & \underline{110.2 $\pm$ 0.9} & \underline{90.6 $\pm$ 1.2} & \underline{74.7 $\pm$ 2.9} & \underline{49.5 $\pm$ 0.1}\\
2 & $\alpha_{tanh}$ & 43.9 $\pm$ 0.8 & 27.2 $\pm$ 0.5 & 115.3 $\pm$ 0.5 & 91.3 $\pm$ 0.4 & 87.1 $\pm$ 2.4 & 54.7 $\pm$ 0.8\\
2 & $\alpha_{ident}$ & 89.1 $\pm$ 2.2 & 53.6 $\pm$ 2.9 & 168.1 $\pm$ 2.0 & 113.2 $\pm$ 2.2 & 206.1 $\pm$ 3.5 & 113.6 $\pm$ 5.2\\
\midrule
5 & $\alpha_{sigm}$ & \textbf{\underline{34.7 $\pm$ 0.3}} &  20.1 $\pm$ 0.1 & \textbf{\underline{103.9 $\pm$ 1.8}} & 81.4 $\pm$ 1.1 & \textbf{\underline{66.5 $\pm$ 0.6}} & \underline{40.4 $\pm$ 3.1}\\
5 & $\alpha_{soft}$ & 37.2 $\pm$ 0.3 & \underline{19.4 $\pm$ 0.1} & 106.4 $\pm$ 0.8 & 82.5 $\pm$ 1.2 & 69.1 $\pm$ 0.3 & 42.0 $\pm$ 2.0\\
5 & $\alpha_{tanh}$ & 39.4 $\pm$ 1.1 & 20.0 $\pm$ 0.1 & 107.2 $\pm$ 0.8 & \underline{80.8 $\pm$ 0.6} & 70.3 $\pm$ 1.3 & 42.8 $\pm$ 0.5\\
5 & $\alpha_{ident}$ & 61.2 $\pm$ 0.3 & 37.3 $\pm$ 0.2 & 127.9 $\pm$ 0.4 & 97.5 $\pm$ 2.8 & 135.9 $\pm$ 1.1 & 77.5 $\pm$ 2.0\\
\midrule
10 & $\alpha_{sigm}$ & 38.9 $\pm$ 3.0 & 18.0 $\pm$ 0.1 & \underline{110.2 $\pm$ 1.7} & 79.0 $\pm$ 0.7 & 68.4 $\pm$ 0.1 & 34.8 $\pm$ 1.2\\
10 & $\alpha_{soft}$ & \underline{36.2 $\pm$ 0.9} & \textbf{\underline{17.1 $\pm$ 0.2}} & 110.8  $\pm$ 0.4 & 79.2 $\pm$ 0.5 & \underline{67.8 $\pm$ 2.6} & \textbf{\underline{34.5 $\pm$ 0.2}}\\
10 & $\alpha_{tanh}$ & 37.4 $\pm$ 1.2 & 17.4 $\pm$ 0.2 & 112.8 $\pm$ 1.7 & \textbf{\underline{77.7 $\pm$ 0.6}} & 71.0 $\pm$ 1.4 & 34.5 $\pm$ 0.3\\
10 & $\alpha_{ident}$ & 48.7 $\pm$ 0.9 & 28.7 $\pm$ 0.1 & 117.0 $\pm$ 0.2 & 87.1 $\pm$ 1.0 & 91.4 $\pm$ 0.2 & 45.4 $\pm$ 0.1\\
\bottomrule
\end{tabular}
\end{adjustbox}
\end{sc}
\end{small}
\end{center}
\caption{Mean and Minimum FID over 50K iterations on the different datasets.}
\label{mean_min_fid_table}
\end{table*}

\subsection{Generated samples}

The generated samples on each dataset using 1 and 10 discriminators with different $\alpha$ are presented in Figure~\ref{fig:generated_samples_mnist_cifar_celebA}. For an objective assessment of the variety by the end of each iteration, the Intra FID is also provided. We observe the superiority of the generated samples, both in terms of realism and variety, when using $\alpha_{sigm}$, $\alpha_{soft}$, and $\alpha_{tanh}$ on all datasets. However, $\alpha_{tanh}$ seems to show a delayed ability in generating realistic samples, possibly due to the increase of $\alpha$ at a steeper fashion. The inability of generating realistic samples when using $\alpha_{ident}$ is also clearly detected on all datasets, as previously discussed. More importantly, the high variety on the generated set, observed by the high Intra FID, is witnessed since very early iterations when using $\alpha_{sigm}$, $\alpha_{soft}$, and $\alpha_{tanh}$. The observed mitigation of mode collapse is carried out throughout the whole training. 

\begin{figure*}[h]
\centering
\includegraphics[width=1.0\textwidth]{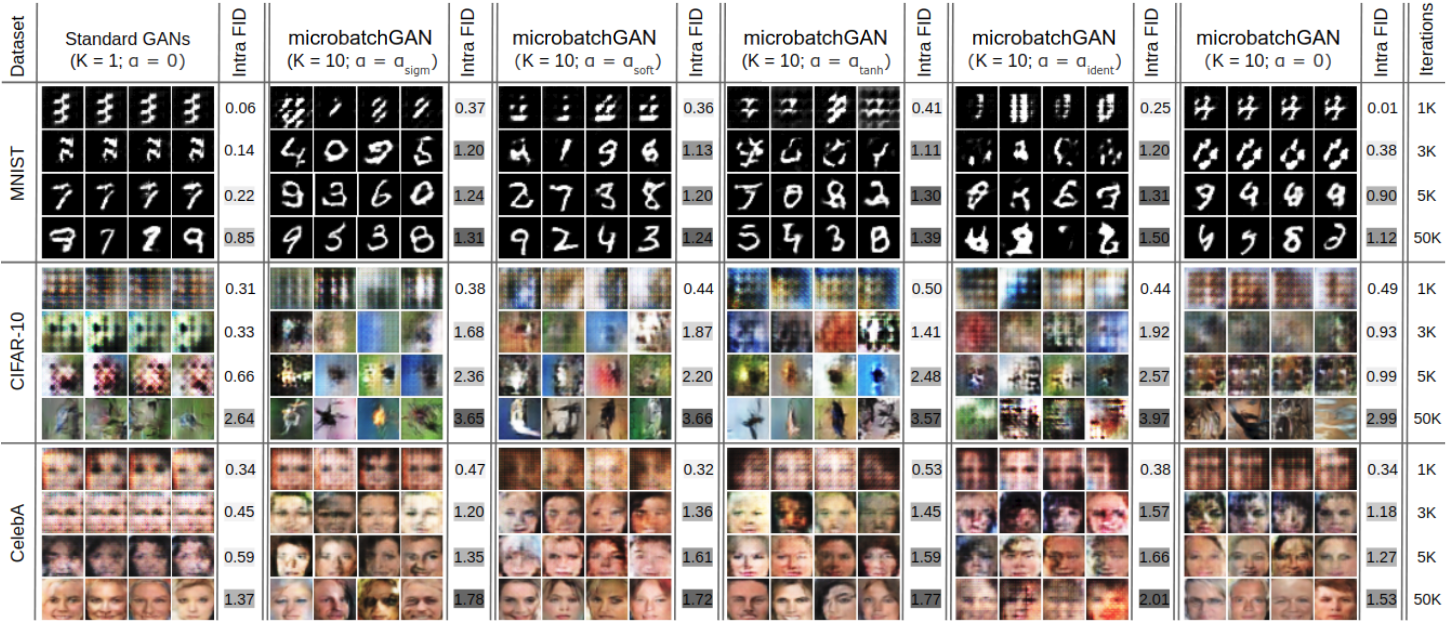}
\caption{Generated samples from 1K, 2K, 5K and 50K iteration with the respective Intra FID.}
\label{fig:generated_samples_mnist_cifar_celebA}
\end{figure*}

When using standard GANs, we notice severe mode collapse, especially early on training. When using 10 discriminators and $\alpha$ set to 0, we notice a slight variation in the generated set, yet, this is only detected after a decent number of iterations, when each $D$ has seen enough samples to guide its judgment to a specific data mode due to the usage of different microbatch for each $D$, delaying sample variety substantially. Thus, using positive $\alpha$ values is shown to be a necessary measure to stimulate variety since the beginning and until the end of training.

\section{Method Comparison}
\label{sec:comparision}

We proceeded to compare different settings of microbatchGAN to other existing methods on 3 different datasets: CIFAR-10, STL-10~\cite{coates2011analysis}, and ImageNet~\cite{russakovsky2015imagenet}. We down-sampled the images of the last two datasets down to 32x32 pixels. We used Inception Score~\cite{is} or IS (higher is better) as the first quantitative metric. Even though IS has been shown to be less correlated with human judgment than FID, most previous works only report results on this metric, making it a useful measure for model comparisons. Out of fairness to the single discriminator methods that we compare our method against, we used only 2 discriminators in our experiments. The architectures and training settings used for all the experiments can be found in the Appendix.

The comparison results are shown in Table~\ref{is_comparision}. We point special attention to the underlined method representing standard GANs, since it was the only method executed with our own implementation and identical training settings as microbatchGAN. Thus, this represents the only method directly comparable to ours. We notice a fair improvement of IS on all the tested datasets, observing an increase up to around 15\% for CIFAR-10, 7\% for STL-10, and 5\% for ImageNet. This indicates the success of our approach on improving the standard GANs framework on multiple datasets with different sizes and challenges.

\begin{table}[h]
\begin{center}
\begin{small}
\begin{sc}
\centering
\begin{adjustbox}{width=0.48\textwidth}
\begin{tabular}{lcccr}
\toprule
 & CIFAR-10 & STL-10 & ImageNet\\
\midrule
Real data    &  11.24 & 26.08 & 25.78 \\
\midrule
WGAN~\cite{wgan} & 3.82 & - & -\\
MIX+WGAN~\cite{generalization_and_equilibrium_in_gans} & 4.04 & - & -\\
ALI~\cite{ali} & 5.34 & - & -\\
BEGAN~\cite{berthelot2017began} & 5.62 & - & -\\
MAGAN~\cite{wang2017magan} & 5.67 & - & -\\
GMAN (K = 2)~\cite{gman} & 5.87 & - & -\\
\underline{GANs*~\cite{gans}} & \underline{5.92} & \underline{6.78} & \underline{7.04}\\
Dropout-GAN (K = 2)~\cite{mordido2018dropout} & 5.98 & - & -\\
GMAN (K = 5)~\cite{gman} & 6.00 & - & -\\
Dropout-GAN (K = 5)~\cite{mordido2018dropout} & 6.05 & - & -\\
DCGAN~\cite{dcgan} & 6.40  & 7.54 & 7.89\\
Improved-GAN~\cite{is} & 6.86 & - & -\\
D2GAN~\cite{dual_GANs} & 7.15 & 7.98 & 8.25\\
DFM~\cite{warde2016improving} & 7.72 & 8.51 & 9.18\\
MGAN~\cite{multi_generator_gans} & 8.33 & 9.22 & 9.32 \\
\midrule
microbatchGAN ($K = 2; \alpha = \alpha_{sigm}$) & 6.77 & 7.23 & 7.32 \\
microbatchGAN ($K = 2; \alpha = \alpha_{soft}$) & 6.66 & 7.19 & 7.40 \\
microbatchGAN ($K = 2; \alpha = \alpha_{tanh}$) & 6.61 & 7.07 & 7.40 \\
\bottomrule
\end{tabular}
\end{adjustbox}
\end{sc}
\end{small}
\end{center}
\caption{Inception scores. For a fair comparison, only unsupervised methods are compared.}
\label{is_comparision}
\end{table}

On CIFAR-10, microbatchGAN achieves competitive results, significantly outperforming GMAN with 5 discriminators while using a similar architecture. We argue that the use of more powerful architectures in the higher ranked methods plays a big role in their end score, especially for DCGAN. Nonetheless, we acknowledge that using different objectives for each $D$ (as proposed in D2GAN) seems to be beneficial in a multi-discriminator setting, representing a good path to follow in the future. Moreover, we observe that using extra autoencoders (DFM) or classifiers (MGAN) in the framework can help to achieve a better performance in the end. However, we note that MGAN makes use of a 10 generator framework, on top of an extra classifier, to achieve the presented results. 
Furthermore, the generated samples presented in their paper (\cite{multi_generator_gans}) indicate signs of partial mode collapse, which is not reflected in its high IS. 


\begin{table}[h]
\begin{center}
\begin{small}
\begin{sc}
\centering
\begin{tabular}{lcccr}
\toprule
  & CIFAR-10\\
\midrule
GANs~\cite{gans}  & 70.23\\
mod-GANs~\cite{gans} & 79.58\\
LSGAN~\cite{mao2017least}  & 83.66\\
DRAGAN~\cite{kodali2017convergence}  & 80.57\\
\midrule
GANs ($K = 2$)  & 74.07\\
mod-GANs ($K = 2$) & 71.96\\
LSGAN ($K = 2$)  & 73.33\\
DRAGAN ($K = 2$)  & 75.83\\
Dropout-GANs ($K = 2$) & 66.82\\
Dropout-mod-GANs ($K = 2$) & 67.57\\
Dropout-LSGAN ($K = 2$) & 69.37\\
Dropout-DRAGAN ($K = 2$) & 66.90\\
\midrule
microbatchGAN ($K = 2; \alpha = \alpha_{sigm}$)  & 66.93\\
microbatchGAN ($K = 2; \alpha = \alpha_{soft}$)  & 65.54\\
microbatchGAN ($K = 2; \alpha = \alpha_{tanh}$)  & 65.84\\
\bottomrule
\end{tabular}
\end{sc}
\end{small}
\end{center}
\caption{Minimum FID comparison.}
\label{fid_comparision}
\end{table}

We further compared our best FID with a subset of the reported methods in~\cite{are_all_created_gans_equal}, namely GANs, both with the original and modified objective, LSGAN, and DRAGAN on CIFAR-10. These methods were chosen since they represent interesting variants of standard GANs, as presented in~\cite{are_all_created_gans_equal}.

\begin{figure}[h]
  \begin{center}
    \includegraphics[width=0.48\textwidth]{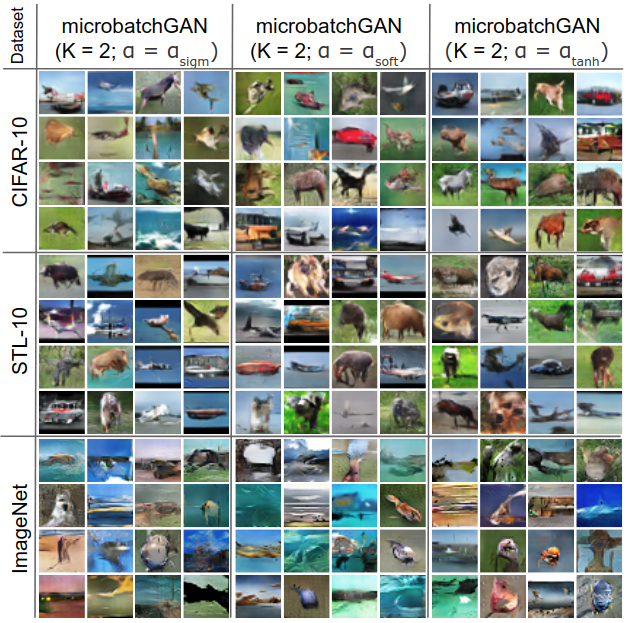}
  \end{center}
  \caption{CIFAR-10, STL-10, and ImageNet results.}
  \label{fig:cifar_stl_imagenet_results}
\end{figure}

We extended each method to an ensemble of discriminators, for a fair comparison to our multiple discriminator approach. Furthermore, we compare against additional results with adversarial dropout at a dropout rate of $0.5$, as proposed in \cite{mordido2018dropout}. We used the same architecture of the last experiment for all methods. Results are shown in Table~\ref{fid_comparision}. We observe that all variants of microbatchGAN outperform the rest of the compared methods under controlled and equal experiments.

A subset of the generated samples produced by the different variations of microbatchGAN reported in Table~\ref{is_comparision} are shown in Figure~\ref{fig:cifar_stl_imagenet_results}, where we observe high variety and realism across all generated sets. Extended results are provided in the Appendix.

\section{Conclusions}

In this work, we present a novel framework, named microbatchGAN, where each $D$ performs microbatch discrimination, differentiating between samples within and outside its fake microbatch. This behavior is enforced by the diversity parameter $\alpha$, that is indirectly self-learned by $G$. In the first iterations, $G$ increases $\alpha$ to lower its loss, expanding its output. Then, as $\alpha$ gradually saturates and each $D$ learns how to better distinguish between real and fake samples, $G$ is forced to fool each $D$ by promoting realism in its output, while keeping the diversity in the generated set. We show evidence that our solution produces realistic and diverse samples on multiple datasets of different sizes and nature, ultimately mitigating mode collapse.

{\small
\bibliographystyle{ieee}
\bibliography{egbib}
}

\newpage
\newpage\null\thispagestyle{empty}\newpage
\appendix

\section{Training settings}

The architectural and training settings used in Sections~\ref{sec:solution},~\ref{sec:results}, and~\ref{sec:comparision} are presented in Tables~\ref{arc_toy}, \ref{arc_mnist_cifar_celeba}, and~\ref{arc_cifar_stl_imagenet}, respectively. For the FID comparison on CIFAR-10 and CelebA in Section~\ref{sec:comparision}, we used the same architectures as Table~\ref{arc_cifar_stl_imagenet} but with a batch size of 64 on both datasets, and ran for 78K iterations on CIFAR-10 and 125K iterations on CelebA.

\begin{table*}[b]
\caption{Training settings for the toy dataset.}
\label{arc_toy}
\begin{center}
\begin{small}
\begin{sc}
\begin{tabular}{lcccccr}
\toprule
 & Feature maps & Nonlinearity\\
\midrule
$G(z): z \sim Normal(0,I)$    & 256 & \\
Fully connected    & 128 & ReLu\\
Fully connected    & 128 & ReLu\\
Fully connected    & 2 & Linear\\
\midrule
$D(x)$    & 2 &\\
Fully connected    & 128 & ReLu\\
Fully connected    & 1 & Softplus\\
\midrule
Number of discriminators & 8 & &  &  & \\
$\alpha$ (static) & $\{0, 0.1, 0.2, 0.3, 0.4, 0.5, 0.6, 0.7, 0.8, 0.9, 1.0\}$ & &  &  & \\
$\alpha$ (self-learned) & $\{\alpha_{sigm}, \alpha_{soft}, \alpha_{tanh}, \alpha_{ident}\}$  & &  &  & \\
Batch size & 512 & & & & \\
Iterations & 25K & & & & \\
Optimizer & Adam ($lr = 0.0002, \beta_1 = 0.5$) & & & & \\
\bottomrule
\end{tabular}
\end{sc}
\end{small}
\end{center}
\end{table*}

\begin{table*}[t]
\caption{Training settings for MNIST, CIFAR-10, and CelebA.}
\label{arc_mnist_cifar_celeba}
\begin{center}
\begin{small}
\begin{sc}
\centering
\begin{adjustbox}{width=1\textwidth}
\begin{tabular}{lcccccr}
\toprule
 & Kernel & Strides & Feature maps & Batch Norm. & Nonlinearity\\
\midrule
$G(z): z \sim Uniform[-1,1]$    & - & - & 100 & - & -\\
Transposed convolution    & $3\times3$ & $4\times4$ & 128 & Yes & ReLu\\
Transposed convolution    & $5\times5$ & $2\times2$ & 64 & Yes & ReLu\\
Transposed convolution    & $5\times5$ & $2\times2$ & 32 & Yes & ReLu\\
Transposed convolution    & $5\times5$ & $2\times2$ & 1/3 & No & Tanh\\
\midrule
$D(x)$    & - & - & $32\times32\times1$/$3$ & - & -\\
Convolution    & $3\times3$ & $2\times2$ & 32 & Yes & Leaky ReLu (0.2)\\
Convolution    & $3\times3$ & $2\times2$ & 64 & Yes & Leaky ReLu (0.2)\\
Convolution    & $3\times3$ & $2\times2$ & 128 & Yes & Leaky ReLu (0.2)\\
Fully connected    & - & - & 1 & No & Sigmoid\\
\midrule
Number of discriminators & $\small\{1, 2, 5, 10\small\}$ & &  &  & \\
$\alpha$ (static) & $\small\{0\small\}$ & &  &  & \\
$\alpha$ (self-learned) & $\{\alpha_{sigm}, \alpha_{soft}, \alpha_{tanh}, \alpha_{ident}\}$ & &  &  & \\
Batch size & 100 & & & & \\
Iterations & 50K & & & & \\
Optimizer & Adam ($lr = 0.0002, \beta_1 = 0.5$) & & & & \\
\bottomrule
\end{tabular}
\end{adjustbox}
\end{sc}
\end{small}
\end{center}
\end{table*}

\begin{table*}[t]
\caption{Training settings for CIFAR-10, STL-10, and ImageNet.}
\label{arc_cifar_stl_imagenet}
\begin{center}
\begin{small}
\begin{sc}
\centering
\begin{adjustbox}{width=1\textwidth}
\begin{tabular}{lcccccr}
\toprule
 & Kernel & Strides & Feature maps & Batch Norm. & Nonlinearity\\
\midrule
$G(z): z \sim Uniform[-1,1]$    & - & - & 100 & - & -\\
Transposed convolution    & $3\times3$ & $4\times4$ & 256 & Yes & ReLu\\
Transposed convolution    & $5\times5$ & $2\times2$ & 128 & Yes & ReLu\\
Transposed convolution    & $5\times5$ & $2\times2$ & 64 & Yes & ReLu\\
Transposed convolution    & $5\times5$ & $2\times2$ & 1/3 & No & Tanh\\
\midrule
$D(x)$    & - & - & $32\times32\times1$ & - & -\\
Convolution    & $3\times3$ & $2\times2$ & 64 & Yes & Leaky ReLu (0.2)\\
Convolution    & $3\times3$ & $2\times2$ & 128 & Yes & Leaky ReLu (0.2)\\
Convolution    & $3\times3$ & $2\times2$ & 256 & Yes & Leaky ReLu (0.2)\\
Fully connected    & - & - & 1 & No & Sigmoid\\
\midrule
Number of discriminators & $\small\{2\small\}$ & &  &  & \\
$\alpha$ (self-learned) & $\{\alpha_{sigm}, \alpha_{soft}, \alpha_{tanh}\}$ & &  &  & \\
Batch size & 100 & & & & \\
Iterations & 200K, 400K, 1M & & & & \\
Optimizer & Adam ($lr = 0.0002, \beta_1 = 0.5$) & & & & \\
\bottomrule
\end{tabular}
\end{adjustbox}
\end{sc}
\end{small}
\end{center}
\end{table*}

\section{Sigmoid initial value}

In Figure~\ref{fig:sigmoid_toy_dataset}, we show and discuss the effects of using different $\beta_{sigm}$ on $\alpha_{sigm}$ on the toy dataset, giving more insights regarding the choice of $\beta_{sigm} = -1.8$ mentioned in Section~\ref{sec:solution}.

\begin{figure*}[h]
\subfigure[Generated samples.]{\includegraphics[width=.49\textwidth]{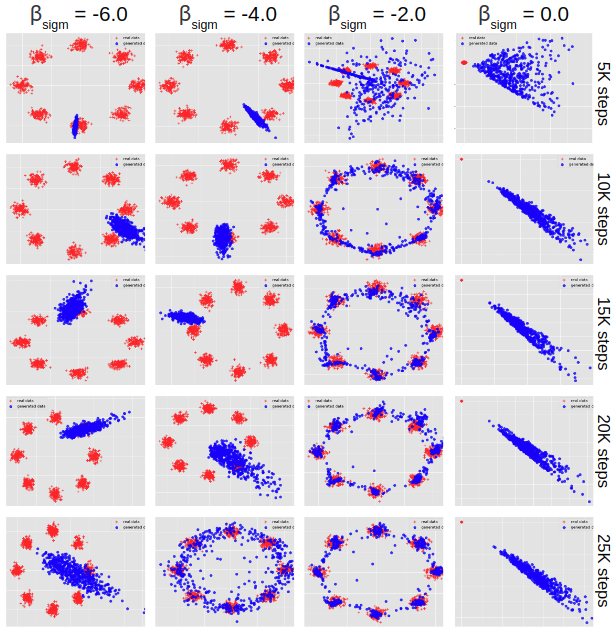}}
\begin{raggedleft}
  \subfigure[$\alpha$ evolution.]{\includegraphics[width=.49\textwidth]{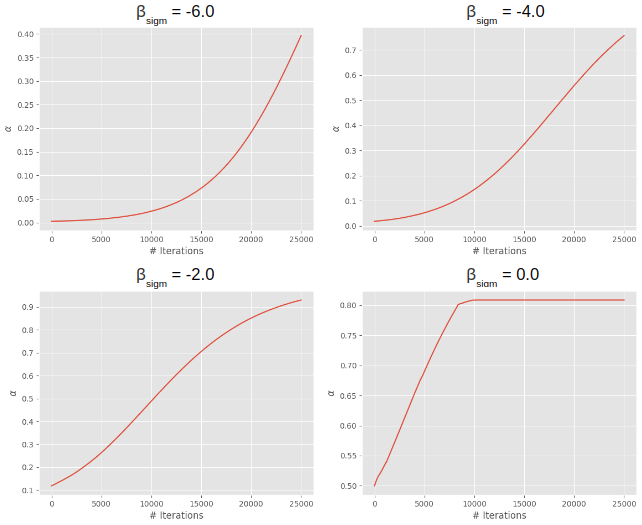}}
  \end{raggedleft}
\caption{Analysis of self-learning $\alpha_{sigm}$ with different initial values of $\beta$. The generated samples in (a) show that using lower $\beta_{sigm}$ values lead the model to mode collapse, since only low $\alpha$ values are used throughout the whole training. On the other hand, using higher values, \textit{e.g.,} $\beta_{sigm} = 0.0$, leads to a steeper increase of $\alpha$ values, inducing the model to only generate varied, but not realistic, samples. We empirically found that using $-2.0 \leq \beta_{sigm} \leq -1.8$ led to diverse plus realistic looking samples from early iterations due to the mild, yet meaningful, increase of $\alpha$ throughout training. The evolution of $\alpha$'s values are presented in (b).}
\label{fig:sigmoid_toy_dataset}
\end{figure*}

\section{Toy dataset comparisons}

Figure~\ref{fig:toydataset_comparisons} shows how different methods compare using the above mentioned toy dataset. We compared microbatchGAN's results (K = 8, $\alpha_{sigm}$) to the standard GAN~(\cite{gans}), UnrollledGAN~(\cite{unrolled_gan}), D2GAN~(\cite{dual_GANs}), and MGAN~(\cite{multi_generator_gans}). We observe bigger sample diversity with our method, while still approximating the real data distribution.

\begin{figure*}[h]
\centering
\includegraphics[width=0.7\textwidth]{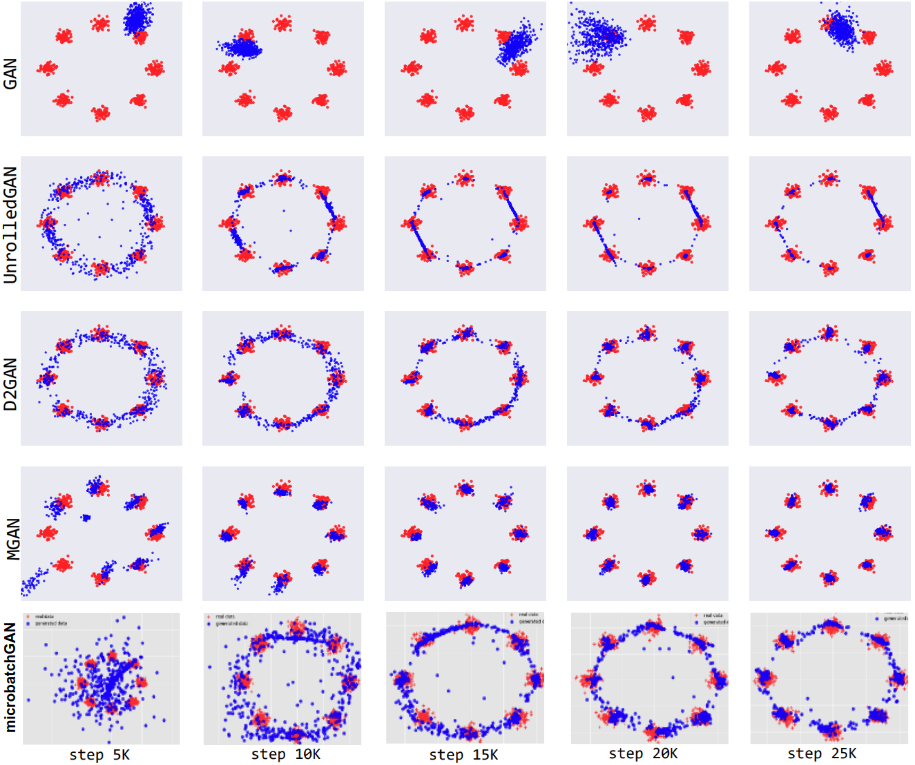}
\caption{Method comparisons on the toy dataset.}
\label{fig:toydataset_comparisons}
\end{figure*}

\section{Extended Results}

Additional results for CIFAR-10, STL-10, and ImageNet are presented bellow.

\begin{figure*}[h]
\centering
\includegraphics[width=0.4\textwidth]{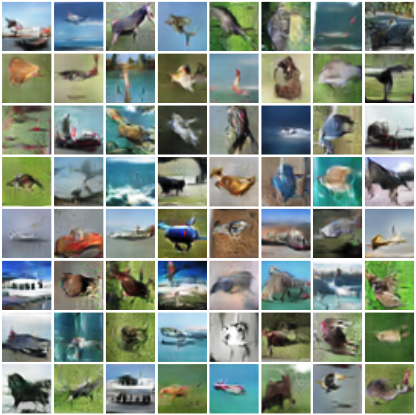}
\caption{CIFAR-10 extended results using K = 2 and $\alpha_{sigm}$.}
\label{fig:cifar_extended_sigm}
\end{figure*}

\begin{figure*}[h]
\centering
\includegraphics[width=0.4\textwidth]{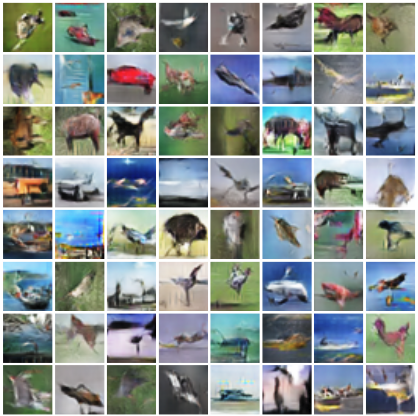}
\caption{CIFAR-10 extended results using K = 2 and $\alpha_{soft}$.}
\label{fig:cifar_extended_soft}
\end{figure*}

\begin{figure*}[h]
\centering
\includegraphics[width=0.4\textwidth]{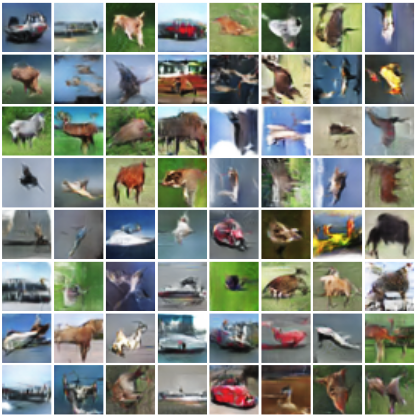}
\caption{CIFAR-10 extended results using K = 2 and $\alpha_{tanh}$.}
\label{fig:cifar_extended_tanh}
\end{figure*}

\begin{figure*}[h]
\centering
\includegraphics[width=0.4\textwidth]{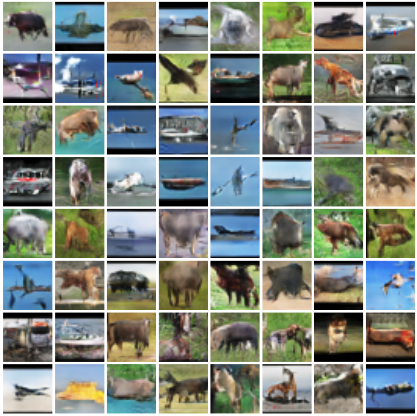}
\caption{STL-10 extended results using K = 2 and $\alpha_{sigm}$.}
\label{fig:stl_extended_sigm}
\end{figure*}

\begin{figure*}[h]
\centering
\includegraphics[width=0.4\textwidth]{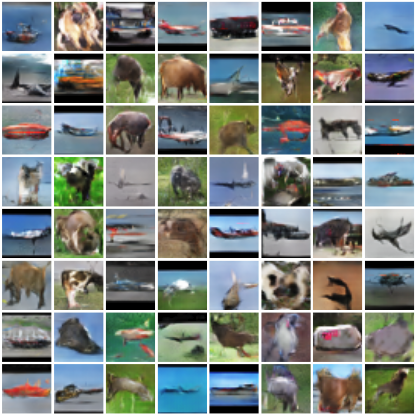}
\caption{STL-10 extended results using K = 2 and $\alpha_{soft}$.}
\label{fig:stl_extended_soft}
\end{figure*}

\begin{figure*}[h]
\centering
\includegraphics[width=0.4\textwidth]{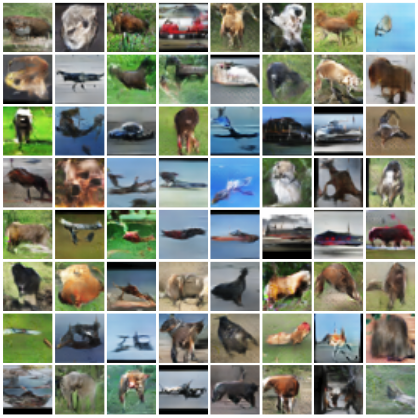}
\caption{STL-10 extended results using K = 2 and $\alpha_{tanh}$.}
\label{fig:stl_extended_tanh}
\end{figure*}

\begin{figure*}[h]
\centering
\includegraphics[width=0.4\textwidth]{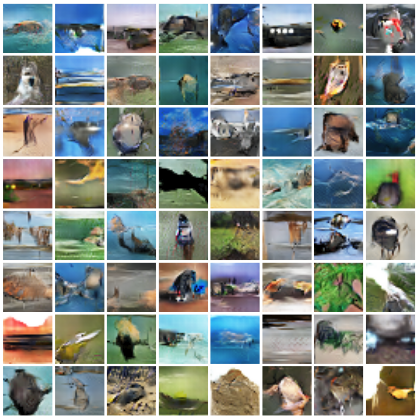}
\caption{ImageNet extended results using K = 2 and $\alpha_{sigm}$.}
\label{fig:imagenet_extended_sigm}
\end{figure*}

\begin{figure*}[h]
\centering
\includegraphics[width=0.4\textwidth]{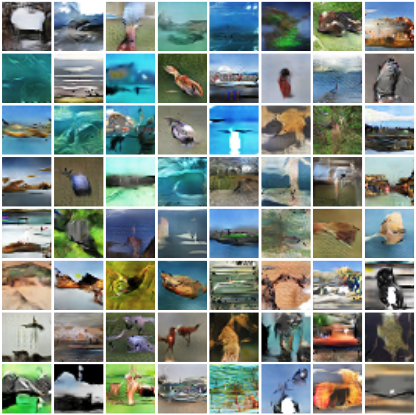}
\caption{ImageNet extended results using K = 2 and $\alpha_{sigm}$.}
\label{fig:imagenet_extended_soft}
\end{figure*}

\begin{figure*}[h]
\centering
\includegraphics[width=0.4\textwidth]{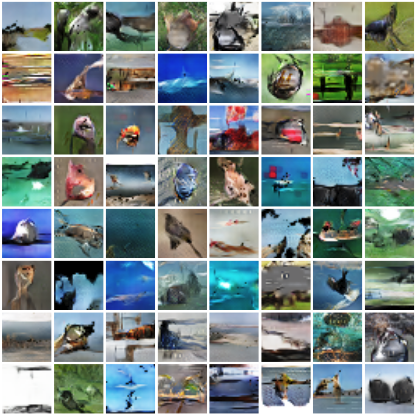}
\caption{ImageNet extended results using K = 2 and $\alpha_{sigm}$.}
\label{fig:imagenet_extended_tanh}
\end{figure*}


\end{document}